\documentclass[10pt,journal,compsoc]{IEEEtran}

\ifCLASSOPTIONcompsoc
  \usepackage[nocompress]{cite}
\else
  \usepackage{cite}
\fi
\ifCLASSINFOpdf
\else
\fi

\usepackage{amsthm}
\usepackage{amsmath,amssymb}
\usepackage{amsfonts}
\usepackage{commath}
\usepackage{graphicx,subfigure}
\usepackage{float}
\usepackage{graphics}
\usepackage{booktabs}
\usepackage{multirow}
\usepackage{bm}
\usepackage{cite}
\usepackage{caption}
\usepackage[usenames, dvipsnames]{color}
\usepackage[table, svgnames]{xcolor}
\usepackage{array}
\usepackage[ruled,vlined]{algorithm2e}
\def\BibTeX{{\rm B\kern-.05em{\sc i\kern-.025em b}\kern-.08em
		T\kern-.1667em\lower.7ex\hbox{E}\kern-.125emX}}
\usepackage{graphicx} 
\usepackage{bm}
\newcommand{\EQ}{\begin{eqnarray}}
	\newcommand{\EN}{\end{eqnarray}}

\usepackage{tabulary}
\newtheorem{thm}{Theorem}

\newcommand*{\rowstyle}[1]{
	\gdef\@rowstyle{\leavevmode#1}%
	\@rowstyle\ignorespaces}

\begin{document}
%
\title{Analytic Learning of Convolutional Neural Network For Pattern Recognition}
%
%
%
%

\author{Huiping~Zhuang,~\IEEEmembership{Member,~IEEE,}
        Zhiping~Lin,~\IEEEmembership{Senior Member,~IEEE,}
        Yimin~Yang,~\IEEEmembership{Senior Member,~IEEE,}
        and~Kar-Ann~Toh$^{*}$,~\IEEEmembership{Senior Member,~IEEE,}
\IEEEcompsocitemizethanks{\IEEEcompsocthanksitem Huiping Zhuang and Zhiping Lin are with the School of Electrical and Electronic Engineering, Nanyang Technological University, Singapore.\protect\\
E-mail: huiping001@e.ntu.edu.sg; ezplin@ntu.edu.sg.
\IEEEcompsocthanksitem Yimin Yang 
is with the Computer Science Department, Lakehead University,
Thunder Bay, Ontario P7B 5E1, Canada.\protect\\
E-mail: yyang48@lakeheadu.ca. 
\IEEEcompsocthanksitem Kar-Ann Toh is with the Department of Electrical and Electronic Engineering, Yonsei University.\protect\\
E-mail: katoh@yonsei.ac.kr. 
}
\thanks{Corresponding author: Kar-Ann Toh.}}

\IEEEtitleabstractindextext{%
\begin{abstract}
Training convolutional neural networks (CNNs) with back-propagation (BP) is time-consuming and resource-intensive particularly in view of the need to visit the dataset multiple times. In contrast, analytic learning attempts to obtain the weights in one epoch. However, existing attempts to analytic learning considered only the multilayer perceptron (MLP). In this article, we propose an analytic convolutional neural network learning (ACnnL). Theoretically we show that ACnnL builds a closed-form solution similar to its MLP counterpart, but differs in their regularization constraints. Consequently, we are able to answer to a certain extent why CNNs usually generalize better than MLPs from the implicit regularization point of view. The ACnnL is validated by conducting classification tasks on several benchmark datasets. It is encouraging that the ACnnL trains CNNs in a significantly fast manner with reasonably close prediction accuracies to those using BP. Moreover, our experiments disclose a unique advantage of ACnnL under the small-sample scenario when training data are scarce or expensive.
\end{abstract}

}

\maketitle

\IEEEdisplaynontitleabstractindextext

%
\IEEEpeerreviewmaketitle

\ifCLASSOPTIONcompsoc
\IEEEraisesectionheading{\section{Introduction}\label{sec:introduction}}
\else
\section{Introduction}
\label{sec:introduction}
\fi
\IEEEPARstart{N}{eural} networks have become increasingly popular owing to the outstanding performance in various complex tasks, e.g., image-based tasks including image classification \cite{imagenet2012}, object detection \cite{rcnn2014,detection2019TPAMI} and image segmentation \cite{segmentation2020TPAMI,unet2015}. However, these remarkable achievements have been built upon an exceedingly amount of training time even with the aid of modern parallel computing tools such as GPUs. A main cause is due to training based on the back-propagation (BP) \cite{bp1974phd} through multiple visits of the dataset. Moreover, such training mechanism often runs into issues related to search convergence. Firstly, updating the network through calculating gradients using BP could encounter vanishing or exploding gradients. Secondly, the nature of iteration-based learning could invite slow convergence given inappropriate initializations. Thirdly, BP learning requires multiple visits (i.e., epochs) of the dataset to acquire the necessary information. One usually needs hundreds of epochs before the training is complete, rendering the learning time-consuming and resource-demanding. Finally, the tunning of hyperparameters such as the learning rate, relies heavily on human experience, and without it the tunning could become very laborious.

These limitations of BP motivate the development of alternative training schemes. Several recent attempts adopted the \textit{analytic learning} \cite{pil2001,karnet2018,cpnet2021} to train the neural network. Unlike the BP training paradigm, these analytic learning methods convert a nonlinear network learning problem into linear segments, which are then tackled by solving matrix equations. Apart from circumventing the aforementioned BP's limitations, the analytic learning also has advantages in network interpretability due to its simple matrix formulation in training \cite{pil2001,interpretable-cnn2019}.

Conventional analytic learning started with shallow networks \cite{analytic-learning-first1991,analytic-learning1992orthogonal,analytic-learning1992inversion,analytic-learning1993}. A good example is the well-known radial basis function (RBF) network \cite{rbf1991univ}. Its parameters in the second layer are trained  using a closed-form least squares (LS) solution after conducting a feature transformation in the first layer. The analytical solutions have also evolved towards training multilayer networks \cite{karnet2018,analytic-learning-finallayer2019noniterative,cpnet2021,brmp2021,kpnet2021training,pil2021,Recomputation2020TPAMI}, receiving an increasing attention in recent years. For instance, the correlation projection network (CPNet) \cite{cpnet2021} partitions an multilayer perceptron (MLP) into multiple 2-layer subnetworks which are trained in a locally supervised manner. The most evident benefit of analytic learning for training neural networks is the fast single shot learning for training data of small to medium sizes.

The analytic learning has indeed shown promising fast training speed owing to its one-epoch training style. Yet, current studies can only handle the training of MLP weights, which can be referred to as MLP-based analytic learning. It is a natural motivation to extend such a technique towards more complex networks, such as CNNs. There have been efforts (e.g., \cite{stackcnn2020,analytic-learning-finallayer2019noniterative}) of involving  the CNN training with analytic learning, but only the MLP weights at the final layer are analytically determined. There has not been any successful attempt for analytic learning of CNN weights. Such an analytic learning is challenging due to CNN's sparse structure (e.g., see \cite{sedghi2018the} and section 2.4 \cite{MAIER201986}). The solution to CNN weight computation remains conjectural and primitive \cite{brmp2021}. This has motivated us to explore solutions that push forward the analytic learning for fast training of more complex networks.

In this paper, we propose an \textbf{A}nalytic \textbf{C}onvolutional \textbf{n}eural \textbf{n}etwork \textbf{L}earning (ACnnL) for effective single-shot learning. Unlike existing analytic learning methods, the ACnnL is able to obtain CNN weights using analytical solutions in a one-epoch training manner. This paper contributes as follows.
\begin{itemize}
	\item To the best of our knowledge, the proposed ACnnL is the first analytic algorithm for training CNNs. It obtains CNN weights in closed-form while completing the training in a single epoch.
	\item We demonstrate that the CNN structure can be treated as that of a generalized MLP, but differs in that the CNN is regularized by significantly more constraints. In this regard, we can justify to a certain extent why CNNs generalize better than MLPs (i.e., CNNs are heavily regularized).
	
	\item In comparison with its BP counterpart, the ACnnL completes the training with a significantly faster speed (e.g., $>100\times$) and with a reasonably close generalization performance.
	
	\item On several benchmark datasets, the ACnnL is shown to significantly outperform the MLP-based analytic learning methods in terms of generalization. In particular, for small-sample datasets, the ACnnL is shown to excel in prediction comparing with that of BP. That is, the ACnnL is less prone to over-fitting for data-limited scenarios naturally due to the well-conditioned LS-based solutions.  
\end{itemize}

\section{Related Works}
\subsection{MLP-based Analytic Learning}
Existing methods of analytic learning are mostly restricted to the learning of MLP weights. These MLP-based methods begin at shallow networks (i.e., 2-layer networks) where the weight of the second layer is computed analytically after setting the first layer. For instance, the RBF network \cite{rbf1991univ} is mainly characterized by the kernel projection in the first layer. Random-weighted networks \cite{random-first1992,pil-first1995exact,pil2001}, another category of shallow networks, are identified by the random weight selection in the first layer. 

Extension of analytic learning to multilayer/deep networks has also been fruitful. The main difficulty behind analytic learning in multilayer networks is to deal with the learnings in hidden layers. One particular solution is to project the target label information into the hidden layers such that these hidden layers can guide their own supervised learnings.  In \cite{karnet2018}, a kernel-and-range network (KARnet) projects the label information into hidden layers utilizing a sequence of Moore-Penrose (MP) inverses \cite{MP-inverse-book2003}. However, the KARnet requires the activation function to be invertible. Its generalization performance is less ideal especially on datasets with large sample size. The random-weighted networks are also extended to multiple layers \cite{elm-multilayer2017multiple} with a constraint of utilizing the invertible activation function like that in the KARnet. The correlation projection network (CPNet) in \cite{cpnet2021} adopts an alternative projection strategy named label encoding. The combination of correlation projection and the label encoding technique allows the CPNet to become one of the MLP-based state-of-the-art analytic learning methods.

\subsection{Analytic Learning Beyond the MLP Structure}
There have also been attempts and prototypes that incorporate analytic learning beyond the MLPs. Most of these methods are more related to \textit{transfer learning} where the non-MLP networks are pre-trained elsewhere and used as feature extractors. The RBM-GI \cite{analytic-learning-finallayer2019noniterative} adopts a feature extractor that has been pre-trained using the restricted Boltzmann machines (RBMs). Features are extracted by the RBMs and mapped to the label with an LS estimator. The deep analytic network (DAN) \cite{stackcnn2020} with transfered features obtains remarkable results. However, aside from needing the feature extractor, the DAN requires fine-tunning of the network with BP to achieved the best performance. There are also methods without adopting pre-trained networks. In \cite{interpretable-cnn2019}, a method named feedforward (FF) design adopts principal component analysis (PCA) to construct the CNN layers which are then followed by MLPs trained using the LS regression. The FF design trains the CNNs with significantly lower cost but expects a considerable loss of generalization performance. 

\noindent\textbf{Other Non-BP Learnings:} There are also other variants related to analytic learning (e.g., non-iterative or non-BP) in the form of stacking structure. The PCAnet \cite{pcanet2012scalable} stacks PCA modules to learn multistage filter banks for extracting features. Like the analytic learning, the deep forest (DF) \cite{deepforest2017} also abandons the BP algorithm to improve the learning efficiency.

In summary, variants of the analytic learning attempts could be deployed across various networks structures. However, only the weights of MLP are obtained analytically due to the difficulty to formulate closed-form solutions beyond the MLP structure. In \cite{brmp2021}, the CNN is treated as a special case of a sparse MLP network, and a prototype of analytic learning for CNN weights is constructed. However, the formulation is primitive due to nontrivial assumptions and constraints. The proposed ACnnL proposed in this paper is able to train CNN weights with closed-form solutions, which closes an important research gap in the literature of analytic learning.


\section{Preliminaries}
Here some background knowledge related to analytic learning is revisited. This mainly includes the LS solution, the CNN structure and its representation.

\subsection{Notations}
For consistency, we unify the notations used in this paper as follows. Without specification, a non-boldface character (e.g., $a$ or $A$) denotes a scalar number. A boldface lowercase character (e.g., $\bm{a}$) represents a column vector. A boldface uppercase character (e.g., $\bm{A}$) indicates a matrix or a tensor\footnote{Also known as multidimensional array.}. $\bm{a}\in\mathbb{R}^{K}$ denotes a column vector $\bm{a}$ with $K$ entries. A 3-dimensional ($3$-D) tensor is represented, for instance, by $\bm{A}\in\mathbb{R}^{N\times K \times T}$ of size $N\times K \times T$. Also in this example, $\bm{A}[n_{1}:n_{2}, k_{1}:k_{2}, :]$ indicates a sliced sub-tensor from $\bm{A}$ with $N,K$ indexes ranging from $n_{1}$ to $n_{2}$ and $k_{1}$ to $k_{2}$ respectively. In particular, in the last dimension $T$, the sub-tensor takes all the entries by marking a colon ``$:$''.

Here we define square bracket $[$ $]$ as the concatenation operator, which concatenates multiple vectors/matrices into one. For instance, let $\bm{A}_{1}\in\mathbb{R}^{a_{1}\times b_{1}}$, $\bm{A}_{2}\in\mathbb{R}^{a_{2}\times b_{2}}$ and $\bm{A}_{3}\in\mathbb{R}^{a_{3}\times b_{3}}$. If $a_{1}=a_{2}=a_{3}$, we can have $\bm{C} = [\bm{A}_{2}, \bm{A}_{2}, \bm{A}_{3}]\in\mathbb{R}^{a_{1}\times (b_{1}+b_{2}+b{3})}$, and if $b_{1}=b_{2}=b_{3}$, we can have $\bm{C} = \begin{bmatrix}
	\bm{A}_{1}\\ \bm{A}_{2}\\ \bm{A}_{3}
\end{bmatrix}\in\mathbb{R}^{(a_{1}+a_{2}+a_{3})\times b_{1}}$.

\begin{figure*}
	\centering
	\includegraphics[width=1\textwidth]{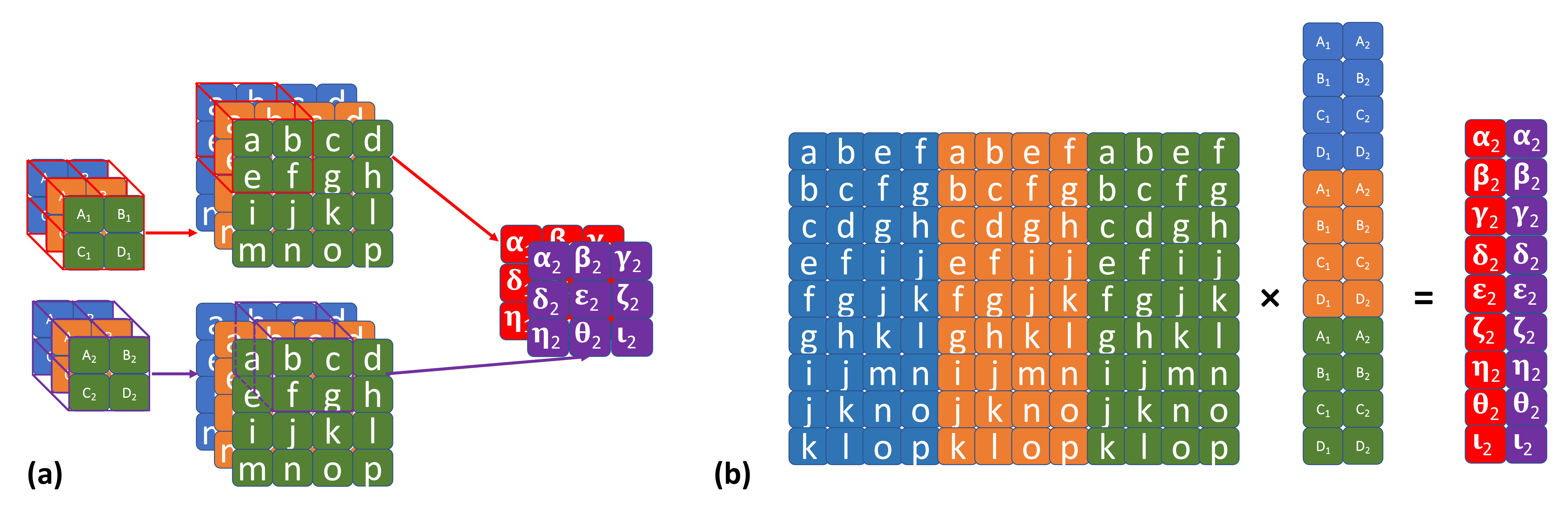}
	\caption{The operation of a convolutional layer.  Traditionally the convolution operation is interpreted through (a) looping the convolution of the block filter with image sub-blocks. (b) The convolution can also be equivalently implemented using a matrix multiplication with the filter weights flattened and the feature maps arranged in matrix form. This is similar to the operation called general matrix multiplication (or image-to-column transformation) \cite{img2col2017ASAP}.}
	\label{fig:conv_matrix}
\end{figure*}
\subsection{Least Square Problem}
Assume that there is a dataset with $N$ data samples $\{x_{n1}, x_{n2}, \dots, x_{nK}, z_{n}\}_{n=1}^{N}$. Specifically, each of the $n^{\text{th}}$ sample has $K$ features with $x_{nk}$ being the $k^{\text{th}}$ feature, and $z_{n}$ being the label/output. Suppose that these features follow a linear combination to represent the label. That is,
\begin{align}\label{eq_ls_scalar}
	z_{n} = w_{1}x_{n1} + w_{2}x_{n2} + \dots + w_{K}x_{nK} + \epsilon_{n}
\end{align}
where $\epsilon_{n}$ is a noise term with Gaussian distribution with zero mean. On the $N$-sample dataset, the formulation in \eqref{eq_ls_scalar} can be written in matrix form:
\begin{align}\label{eq_LR}
	\bm{z} = \bm{X}\bm{w} + \bm{\epsilon}
\end{align}
where $\bm{w}\in\mathbb{R}^{K}$, $\bm{z}\in\mathbb{R}^{N}$ and $\bm{X}\in\mathbb{R}^{N\times K}$ are the weight vector, the label vector, and the data matrix respectively, constructed via
\begin{align*}
	\resizebox{1\linewidth}{!}{$
		\bm{w} = \begin{bmatrix} w_{1}\\w_{2}\\\vdots\\w_{K} \end{bmatrix},
		\bm{z} = \begin{bmatrix} z_{1}\\z_{2}\\\vdots\\z_{N} \end{bmatrix},
		\bm{x}_{n} = \begin{bmatrix} x_{n1}\\x_{n2}\\\vdots\\x_{nK} \end{bmatrix},
		\bm{X} = \begin{bmatrix} \bm{x}_{1}^{T}\\ \bm{x}_{2}^{T}\\ \vdots\\ \bm{x}_{N}^{T} \end{bmatrix}
		\bm{\epsilon} = \begin{bmatrix} \epsilon_{1}\\\epsilon_{2}\\\vdots\\\epsilon_{N} \end{bmatrix}.$}
\end{align*}
The formulation in \eqref{eq_LR} is the model for linear regression. To search for an optimal $\bm{w}$, one minimizes the following cost function:
\begin{align}\label{pro_L2}
	\underset{\bm{w}}{\text{argmin}} \quad \left\lVert\bm{z} - \bm{X}\bm{w}\right\rVert_{2}^{2} + {\gamma} \left\lVert\bm{w}\right\rVert_{2}^{2}
\end{align}
where $\left\lVert\cdot\right\rVert_{2}$ indicates the $l_{2}$-norm. The best estimate $\bm{\hat w}$ for \eqref{pro_L2} is given by:
\begin{align}\label{eq_LS}
	\bm{\hat w} = (\bm{X}^{T}\bm{X} + \gamma\bm{I})^{-1}\bm{X}^{T}\bm{z}
\end{align}
where $\gamma$ is a factor that controls the regularization of $\bm{w}$. The solution in \eqref{eq_LS} is optimal since the cost function $\left\lVert\bm{z} - \bm{X}\bm{\hat w}\right\rVert_{2}^{2} + {\gamma} \left\lVert\bm{\hat w}\right\rVert_{2}^{2}$ is the smallest given any $\bm{w}\in\mathbb{R}^{K}$. For $\gamma\to 0$, \eqref{eq_LS} can be rewritten as
\begin{align}\label{eq_LS_MP}
	\bm{\hat w} = \bm{X}^{\dagger}\bm{z}
\end{align}
where $^{\dagger}$ is the MP inverse operator.

\subsection{Representation of Convolutional Neural Networks}\label{subsection_cnn}
The operation in CNNs is well interpreted through a graphical aid (see Fig. \ref{fig:conv_matrix}(a)), where the convolution is conducted in an iterative manner. Here we put the operation in detail as explicit formulas to prepare for the subsequent development of ACnnL.

We use $\bm{X}_{l}\in\mathbb{R}^{C_{l}\times W_{l} \times H_{l}}$, a $3$-D tensor to represent the input feature map at layer $l$. In detail, $\bm{X}_{l}$ is a feature map having $C_{l}$ channels with  $W_{l}\times H_{l}$ pixels in each channel. Let $\bm{W}^{[j]}_{l}\in\mathbb{R}^{C_{l}\times K_{l}\times K_{l}}$ be the $j^{\text{th}}$ filter weight with kernel size $K_{l}\times K_{l}$ at layer $l$. Also let $\bm{Z}_{l}$ be the output tensor at layer $l$. A convolution operation using filters $\{\bm{W}^{[1]}_{l},\bm{W}^{[2]}_{l},\dots,\bm{W}^{[J]}_{l}\}$ on $\bm{X}_{l}$ yields
\begin{align}\label{eq_conv_original}
	\bm{Z}_{l} = f_{\text{conv}}(\bm{X}_{l},\{\bm{W}^{[1]}_{l},\bm{W}^{[2]}_{l},\dots,\bm{W}^{[J]}_{l}\})
\end{align}
where $f_{\text{conv}}$ denotes the convolution operator. Let $P$ and $S$ indicate the padding and striding operations. After the convolution, we have $\bm{Z}_{l}\in\mathbb{R}^{J\times([(W_{l}-K+2P)S]+1)\times([(H_{l}-K+2P)S]+1)}$. For simplicity, in this paper we adopt zero padding and a stride of one (i.e., $\bm{Z}_{l}\in\mathbb{R}^{J\times(W_{l}-K_{l}+1)\times(H_{l}-K_{l}+1)}$) for illustration purpose. The convolution operation can then be formulated via
\begin{align}\label{eq_conv_loop}
	\bm{Z}_{l}[j,w,h] = \sum_{c=1}^{C_{l}}\bm{X}_{l}[c,w:w+K_{l},h:h+K_{l}]\odot\bm{W}^{j}_{l}[c,:,:]
\end{align}
where $\odot$ is an operator conducting element-wise multiplication between two matrices, with $1\le j\le J$, $1\le w\le (W_{l}-K_{l}+1)$, and $1\le h\le (H_{l}-K_{l}+1)$.


\section{Analytic Convolutional Neural Network Learning}
\begin{figure*}
	\centering
	\includegraphics[width=1\linewidth]{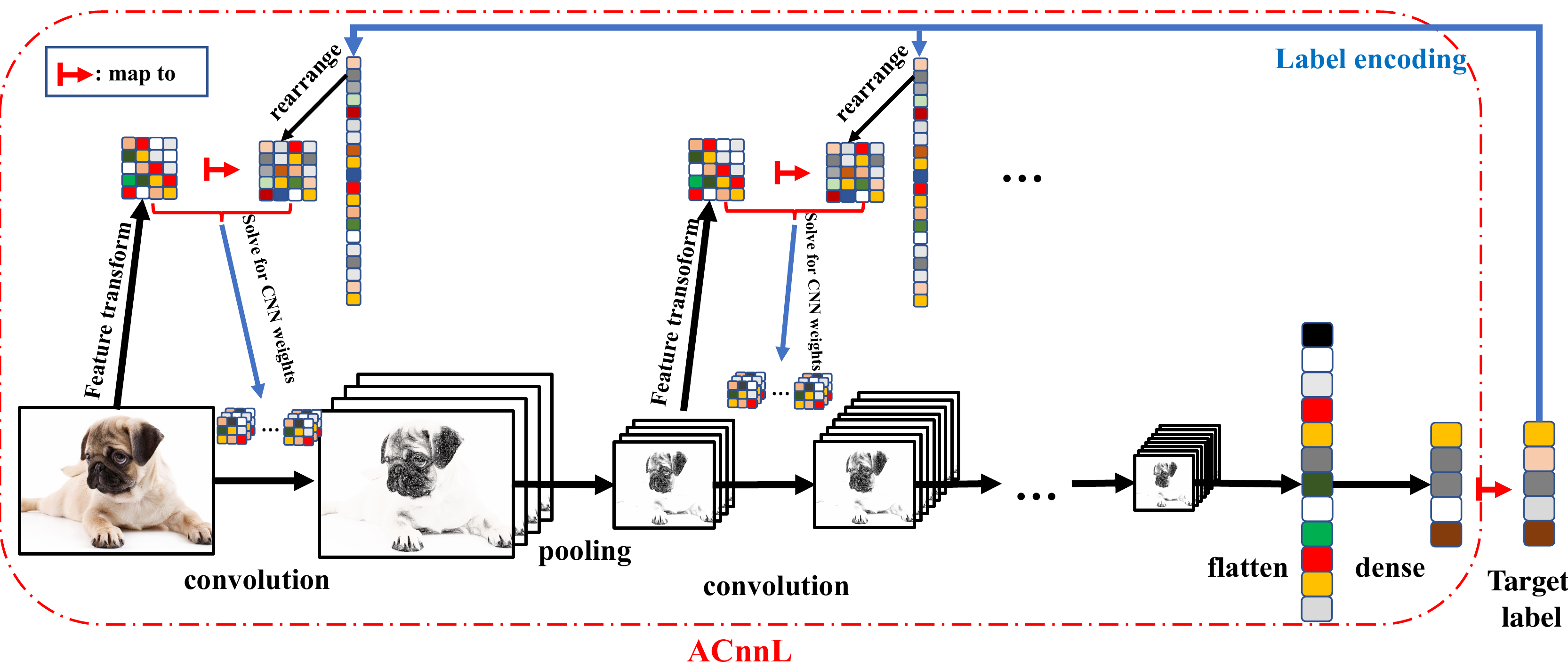}
	\caption{The general structure of the proposed ACnnL. The ACnnL includes several steps. Firstly, the label information is projected into hidden CNN layers using a linear transformation to cope with dimension mismatch. Secondly, for each CNN layer, the input feature map and CNN weights are converted into 2-D matrices according to the feature and weight transformations (e.g., see Fig. \ref{fig:conv_matrix}(b)). Finally, the CNN weights are computed analytically by conducting a linearized locally supervised learning (e.g., see \eqref{eq_encoder}, \eqref{eq_acnn} and \eqref{eq_activation}).}
	\label{fig:acnn-flow}
\end{figure*}

Here we present the proposed ACnnL in details. Essentially, we reformulate the iteration-based interpretation of CNN in Section \ref{subsection_cnn} into a matrix multiplication form. Subsequently, we adopt the label encoding technique \cite{cpnet2021} to introduce the label information in each hidden CNN layer. This allows us to convert the nonlinear CNN learning into a linear one where the CNN weight can be analytically trained. 

To the best of our knowledge, the ACnnL is the first analytic learning algorithm for training CNNs. For simplicity, in this paper we only discuss the vanilla CNN without considering advanced structures (e.g., ResNet \cite{resnet2016}). Specifically, we discuss an $L$-layer vanilla CNN structure containing $L-1$ CNN layers followed by one MLP. The general structure of ACnnL is shown in Fig. \ref{fig:acnn-flow}.
\subsection{Convolution Represented by Matrix Multiplication}\label{subsection_cnn_mtrx}
We shall derive an equivalent matrix multiplication form for the conventional formulation in \eqref{eq_conv_loop} (see Fig. \ref{fig:conv_matrix}(b)). The first step is to combine the set of filter weights $\{\bm{W}^{[1]}_{l},\bm{W}^{[2]}_{l},\dots,\bm{W}^{[J]}_{l}\}$ into a single matrix. To this end, the \textit{flattening operator} $\mathcal{F}$ is defined, which flattens a $m$-D tensor into an ($m-1$)-D one. For instance, for a 3-D tensor $\bm{A}[i,:,:]=\bm{A}_{i}$ ($i=1,2,3,4$), $\mathcal{F}(\bm{A}) = [\bm{A}_{1}\ \bm{A}_{2}\ \bm{A}_{3}\ \bm{A}_{4}]^{T}$. Conversely, $\mathcal{F}^{-1}$ indicates the inversion process, i.e., $\mathcal{F}^{-1}(\mathcal{F}(\bm{A})) = \bm{A}$.

We flatten and concatenate the weight tensors into a 2-D matrix as follows:
\begin{align}\label{eq_w_trans}
	\bm{\mathcal{W}}_{l} = ([\mathcal{F}^{2}(\bm{W}_{l}^{[1]}), \mathcal{F}^{2}(\bm{W}_{l}^{[2]}),\dots,\mathcal{F}^{2}(\bm{W}_{l}^{[J]})])
\end{align}
where $\bm{\mathcal{W}}_{l}\in\mathbb{R}^{(C_{l}K_{l}^{2})\times J}$. The notation $\mathcal{F}^{2}$ is adopted, indicating that the flattening operation is employed twice such that the 3-D tensor $\bm{W}_{l}^{j}$ is brought into a column vector. For convenience, the above weight-flattening operation can be written as $f_{\text{wflat}}$:
\begin{align}\label{eq_w_flatten}
	\bm{\mathcal{W}}_{l} = f_{\text{wflat}}(\bm{W}^{[1]}_{l},\bm{W}^{[2]}_{l},\dots,\bm{W}^{[J]}_{l}).
\end{align}

For feature map $\bm{X}_{l}$, we do similar operations to convert it into a 2-D matrix. This include two steps. Firstly, we flatten the matrices along the channel axis on patches of the feature map , i.e.,  $\bm{X}_{l}[c,w:w+K_{l},h:h+K_{l}]$ in \eqref{eq_conv_loop}, and concatenate them as follows:
\begin{align}\nonumber
	\bm{c}_{l,w,h} = [&
	\mathcal{F}(\bm{X}_{l}[1,w:w+K_{l},h:h+K_{l}])^{T},\\\nonumber
	&\mathcal{F}(\bm{X}_{l}[2,w:w+K_{l},h:h+K_{l}])^{T},
	\dots,\\\label{eq_feature_flatten}
	&\mathcal{F}(\bm{X}_{l}[C_{l},w:w+K_{l},h:h+K_{l}])^{T}
	]
\end{align}
where $\bm{c}_{l,w,h} \in\mathbb{R}^{1\times C_{l}K_{l}^{2}}$ is a row vector. Next, we further concatenate these vectors into a matrix:
\begin{align}\label{eq_feature_flatten2}
	\bm{\mathcal{X}}_{l}
	&=\begin{bmatrix}
		\bm{c}_{l,1,1}\\
		\bm{c}_{l,2,1}\\
		\vdots\\
		\bm{c}_{l,(W_{l}-K_{l}+1),1}\\
		\bm{c}_{l,1,2}\\
		\bm{c}_{l,2,2}\\
		\vdots\\
		\bm{c}_{l,(W_{l}-K_{l}+1),2}\\
		\vdots\\
		\bm{c}_{l, (W_{l}-K_{l}+1)(H_{l}-K_{l}+1)}
	\end{bmatrix}  
\end{align}
where $\bm{\mathcal{X}}_{l}\in\mathbb{R}^{(W_{l}-K_{l}+1)(H_{l}-K_{l}+1)\times C_{l}K_{l}^{2}}$ is a transformed 2-D feature matrix. The above steps can be compactly  denoted  as the following transform $\mathcal{I}$
\begin{align}\label{eq_feature_flatten3}
	\bm{\mathcal{X}}_{l}&= \mathcal{I}(\bm{X}_{l})
\end{align}
which maps the feature map into a 2-D matrix. Conversely, $ \mathcal{I}^{-1}$ performs the opposite operation, i.e., $ \mathcal{I}^{-1}( \mathcal{I}(\bm{X}_{l})) = \bm{X}_{l}$.

Hence, the convolution process in \eqref{eq_conv_loop} can be equivalently converted into a matrix multiplication form as follows:
\begin{align}\nonumber
	\mathcal{F}(\bm{Z}_{l}) &= \mathcal{I}(\bm{X}_{l})f_{\text{wflat}}(\bm{W}^{[1]}_{l},\bm{W}^{[2]}_{l},\dots,\bm{W}^{[J]}_{l})\\\label{eq_cnn_matirx_multiplication}
	&= \bm{\mathcal{X}}_{l}\bm{\mathcal{W}}_{l}
\end{align}
where $\mathcal{F}(\bm{Z}_{l})\in\mathbb{R}^{J\times(W_{l}-K_{l}+1)(H_{l}-K_{l}+1)}$. The output tensor $\bm{Z}_{l}$ identical to that obtained in \eqref{eq_conv_loop} can be adopted by re-arranging the elements through $\mathcal{F}^{-1}(\mathcal{F}(\bm{Z}_{l}))$. The transformation in \eqref{eq_w_trans}-\eqref{eq_cnn_matirx_multiplication} is similar to the operation called general matrix multiplication (or image-to-column transformation) \cite{img2col2017ASAP} implemented by various deep learning platforms (e.g., PyTorch \cite{pytorch2019NeurIPS}) to accelerate the CNN operations.

Owing to the above transformations, the convolution operation can be implemented in matrix product form. Such a representation is in fact a rather practical trick in computers to accelerate the implementation of CNNs. This is because the matrix multiplication can be highly parallelized using acceleration devices such as GPUs. More importantly, we are able to elevate this representation to facilitate analytic learning of the CNN weights, which is formulated in the following subsection.

\subsection{Analytic Learning of CNN Weights}
To formulate an analytic learning of CNN weights, the key is to convert the learning into a linear problem and provide each layer with corresponding label information. Hence, the learning in each hidden layer can be treated as a linear network learning problem where the resultant LS formulation can be solved.
\subsubsection{\bf Label Encoding}
The label information can be projected into hidden layers utilizing the so-called label encoding technique \cite{cpnet2021}. Given an $N$-sample label matrix $\bm{Y}\in\mathbb{R}^{N\times K}$ (e.g., label $Y$ has $K$ classes), pseudo label $\bm{\bar Z}_{l}$ at layer $l$ can be generated via
\begin{align}\label{eq_encoder}
	\bm{\bar Z}_{l} = \bm{Y}\bm{Q}_{l}
\end{align}
where $\bm{Q}_{l}$ is a label encoding matrix generated with each element sampled from a normal distribution. The projection is linear in order to match the dimensionality of $\bm{Z}_{l}$. For instance, if we need to project the label information into a hidden layer with $J(W_{l}-K_{l}+1)(H_{l}-K_{l}+1)$ entries to match $\mathcal{F}(\bm{Z}_{l})$, we have $\bm{Q}_{l}\in\mathbb{R}^{K\times J(W_{l}-K_{l}+1)(H_{l}-K_{l}+1)}$, and $\bm{\bar Z}_{l}\in\mathbb{R}^{N\times J(W_{l}-K_{l}+1)(H_{l}-K_{l}+1)}$. In particular, the projected label can be rearranged by $\mathcal{F}^{-1}(\bm{\bar Z}_{l})$ to have the same shape as that of $\bm{Z}_{l}$.

The encoding does not lose information since we are projecting lower-dimension data into higher ones. This also holds true for classification problems, since the output entries (i.e., number of classes) are normally in smaller number than the number of nodes in hidden layers.
\subsubsection{\bf Linear Formulation and Solution}
By introducing supervised information into hidden layers, we can leverage the label encoding technique to achieve analytic learning. To reach this goal, the matrix representation of CNN in Section \ref{subsection_cnn_mtrx} is the key. The matrix multiplication in \eqref{eq_cnn_matirx_multiplication} allows us to construct a learning problem in a linear manner using the convolution weights $\{\bm{W}^{[1]}_{l},\bm{W}^{[2]}_{l},\dots,\bm{W}^{[J]}_{l}\}$.

\noindent\textbf{Single-sample case}: For simplicity, the single-sample learning case can be formulated as
\begin{align}\label{eq_one}
	\underset{\bm{\mathcal{W}}_{l}}{\mathrm{argmin}}\quad \norm{\bm{\bar Z} -\bm{\mathcal{X}}_{l} \bm{\mathcal{W}}_{l}}_{2}^{2} + {\gamma}\norm{\bm{\mathcal{W}}_{l}}_{2}^{2}.
\end{align}
It is simple to obtain the optimal solution to the linear matrix equation in \eqref{eq_one}, i.e.,
\begin{align}
	\bm{\mathcal{\hat W}}_{l} = (\bm{\mathcal{X}}_{l}^{T}\bm{\mathcal{X}}_{l} + \gamma\bm{I})^{-1}\bm{\mathcal{X}}_{l}^{T}\bm{\bar Z}_{l}.
\end{align}
The above single-sample example is to straightforwardly demonstrates that the CNN filtering weights can be calculated by solving the linear matrix regression problem. Such a layer-wise learning has been proved effectively in various existing analytic learning methods. The more general version (i.e., multiple-sample case) can be extended to in a similar manner, but with additional efforts as shown below.

\noindent\textbf{Multiple-sample case}: To fit multiple samples, the optimization in \eqref{eq_one} must be expanded to
\begin{align}\label{eq_multiple}
	\underset{\bm{\mathcal{W}}_{l}}{\mathrm{argmin}}\quad \sum_{n=1}^{N}\norm{\bm{\bar Z}^{(n)} - \bm{\mathcal{X}}_{l}^{(n)}\bm{\mathcal{W}}_{l}}_{2}^{2} + {\gamma}\norm{\bm{\mathcal{W}}_{l}}_{2}^{2}
\end{align}
where the superscript $[\cdot]^{(n)}$ indicates the $n^{\text{th}}$ sample of the total $N$ samples.

The multiple-sample optimization problem in \eqref{eq_multiple} appears to be of significant modification based on the single-sample case. However, it only needs a fair amount of additional work to obtain the solution. We summarize it in the following Theorem.

\begin{thm}
	The CNN weight matrix $\bm{\mathcal{W}}_{l}$ in \eqref{eq_multiple} that fits multiple training samples $\{\{\bm{\mathcal{X}}_{l}^{(1)},\bm{\bar Z}_{l}^{(1)}\},\{\bm{\mathcal{X}}_{l}^{(2)},\bm{\bar Z}_{l}^{(2)}\},\dots,\{\bm{\mathcal{X}}_{l}^{(N)},\bm{\bar Z}_{l}^{(N)}\}\}$ can be obtained optimally in least squares sense as
	\begin{align}\label{eq_acnn}
		\bm{\mathcal{\hat W}}_{l} =\left(\sum_{n=1}^{N}\bm{\mathcal{X}}_{l}^{(n)T}\bm{\mathcal{X}}_{l}^{(n)}+\gamma\bm{I}\right)^{-1}\left(\sum_{n=1}^{N}\bm{\mathcal{X}}_{l}^{(n)T}\bm{\bar Z}_{l}^{(n)}\right).
	\end{align}
\end{thm}
\begin{proof}
	We take derivatives with respect to (w.r.t.) $\bm{\mathcal{W}}_{l}$ on the objective function in \eqref{eq_multiple}, which leads to
	\begin{align}\nonumber
		&\frac{\partial \left(\sum\limits_{n=1}^{N}\norm{\bm{\bar Z}^{(n)} - \bm{\mathcal{X}}_{l}^{(n)}\bm{\mathcal{W}}_{l}}_{2}^{2} + {\gamma}\norm{\bm{\mathcal{W}}_{l}}_{2}^{2}\right)}{\partial \bm{\mathcal{W}}_{l}}\\\nonumber
		&= \sum\limits_{n=1}^{N} \frac{\partial \left(\norm{\bm{\bar Z}^{(n)} - \bm{\mathcal{X}}_{l}^{(n)}\bm{\mathcal{W}}_{l}}_{2}^{2} \right)}{\partial \bm{\mathcal{W}}_{l}}+ 2{\gamma}\bm{\mathcal{W}}_{l}\\\nonumber
		&=-2\sum\limits_{n=1}^{N}\bm{\mathcal{X}}_{l}^{(n)T}(\bm{\bar Z}^{(n)} - \bm{\mathcal{X}}_{l}^{(n)}\bm{\mathcal{W}}_{l}) + 2\gamma\bm{\mathcal{W}}_{l}\\\label{eq_derivative}
		&=   2\big(\sum_{n=1}^{N}\bm{\mathcal{X}}_{l}^{(n)T}\bm{\mathcal{X}}_{l}^{(n)} + \gamma\bm{I}\big)\bm{\mathcal{W}}_{l} - 2\sum_{n=1}^{N}\bm{\mathcal{X}}_{l}^{(n)T}\bm{\bar Z}_{l}^{(n)}.
	\end{align}
	Putting the derivative in \eqref{eq_derivative} to 0 gives 
	\begin{align*}
		&2\big(\sum_{n=1}^{N}\bm{\mathcal{X}}_{l}^{(n)T}\bm{\mathcal{X}}_{l}^{(n)} + \gamma\bm{I}\big)\bm{\mathcal{W}}_{l} - 2\sum_{n=1}^{N}\bm{\bar Z}_{l}^{(n)}\bm{\mathcal{X}}_{l}^{(n)T} =\bm{0}\\
		=>&\big(\sum_{n=1}^{N}\bm{\mathcal{X}}_{l}^{(n)T}\bm{\mathcal{X}}_{l}^{(n)} + \gamma\bm{I}\big)\bm{\mathcal{W}}_{l} = \sum_{n=1}^{N}\bm{\mathcal{X}}_{l}^{(n)T}\bm{\bar Z}_{l}^{(n)}\\
		=>&	\bm{\mathcal{\hat W}}_{l} =\left(\sum_{n=1}^{N}\bm{\mathcal{X}}_{l}^{(n)T}\bm{\mathcal{X}}_{l}^{(n)} + \gamma\bm{I}\right)^{-1}\left(\sum_{n=1}^{N}\bm{\mathcal{X}}_{l}^{(n)T}\bm{\bar Z}_{l}^{(n)}\right) 
	\end{align*}
	which completes the proof.
\end{proof}

For $\gamma\to 0$, we have
\begin{align}\label{eq_acnn_MP}
	\bm{\mathcal{\hat W}}_{l} =\left(\sum_{n=1}^{N}\bm{\mathcal{X}}_{l}^{(n)T}\bm{\mathcal{X}}_{l}^{(n)}\right)^{\dagger}\left(\sum_{n=1}^{N}\bm{\mathcal{X}}_{l}^{(n)T}\bm{\bar Z}_{l}^{(n)}\right).
\end{align}
The form of \eqref{eq_acnn_MP} will be more useful than that of \eqref{eq_acnn} in the following development (e.g., see \eqref{eq_acnn_rewrite}).

After estimating the convolution weights, we feed the input feature map $\bm{\mathcal{X}}_{l}^{(n)}$ to the activation function $g_{l}$ as follows:
\begin{align}\label{eq_activation}
	\bm{\mathcal{X}}_{l+1}^{(n)} = g_{l}(\bm{\mathcal{X}}_{l}^{(n)}\bm{\mathcal{\hat W}}_{l}).
\end{align}
Next, the activations $\{	\bm{\mathcal{X}}_{l+1}^{(1)},	\bm{\mathcal{X}}_{l+1}^{(2)},\dots,	\bm{\mathcal{X}}_{l+1}^{(N)}\}$ serve as inputs to compute the convolution weights at layer $l+1$ according to \eqref{eq_encoder}, \eqref{eq_acnn} and \eqref{eq_activation}. Also, if needed, one can reshape the activation back into the traditional 3-D tensor through
\begin{align}\label{eq_activation2tradition}
	\bm{X}_{l+1}^{(n)} = \mathcal{I}^{-1}(\bm{\mathcal{X}}_{l+1}^{(n)}).
\end{align} 

The training of CNN is conducted in a stacking manner following \cite{cpnet2021,stackcnn2020}. The convolution layers are mainly used for extracting useful features. After completing the training of CNN layers, an MLP layer (e.g., classifier) follows to map the extracted features to the target label.

Firstly, the 3-D feature tensor $\bm{X}_{l}^{(n)}\in\mathbb{R}^{C_{l}\times W_{l}\times H_{l}}$ is flattened into a 2-D matrix $\bm{\mathfrak{X}}$:
\begin{align}\nonumber
	\bm{\mathfrak{X}}_{l} &=f_{\text{wflat}}(\bm{X}_{l}^{(1)},\bm{X}_{l}^{2},\dots,\bm{X}_{l}^{N})\\	\label{eq_flatten_mlp}
	&=([\mathcal{F}^{2}(\bm{X}_{l}^{(1)}), \mathcal{F}^{2}(\bm{X}_{l}^{2}),\dots,\mathcal{F}^{2}(\bm{X}_{l}^{N})])^{T}\in\mathbb{R}^{N\times C_{l}W_{l}H_{l}}.
\end{align}
Subsequently, the flatten feature $\bm{\mathfrak{X}}_{l}$ is mapped onto label $\bm{Y}$ by optimizing
\begin{align}\label{eq_mlp}
	\underset{\bm{\mathcal{W}}_{l}}{\mathrm{argmin}}\quad \norm{\bm{Y} - \bm{\mathfrak{X}}_{l}\bm{\mathcal{W}}_{l}}_{2}^{2} + {\gamma}\norm{\bm{\mathcal{W}}_{l}}_{2}^{2}
\end{align}
where the solution is 
\begin{align}\label{eq_weight_mlp}
	\bm{\mathcal{\hat W}}_{l} = (\bm{\mathfrak{X}}_{l}^{T}\bm{\mathfrak{X}}_{l} + \gamma\bm{I})^{-1}\bm{\mathfrak{X}}_{l}^{T}\bm{Y}.
\end{align}
For $\gamma\to 0$, we have
\begin{align}\label{eq_weight_mlp_MP}
	\bm{\mathcal{\hat W}}_{l} = \bm{\mathfrak{X}}_{l}^{\dagger}\bm{Y}.
\end{align}

For an $L$-layer CNN, the weights of CNN layers or MLP layers can be obtained analytically layer-by-layer in a stacking fashion. Owing to the closed-form solutions, the training starts from the first layer and ends at the last without visiting the dataset more than once. That is, the ACnnL trains CNNs with merely one epoch, consistent with the requirements met by existing analytic learning methods \cite{interpretable-cnn2019,cpnet2021,pil-first1995exact}. We summarize the ACnnL in Algorithm \ref{algo_An_CNN}.

\begin{algorithm}
	\SetAlgoLined
	\textbf{Inputs:} {\small input samples $\{\bm{{X}}_{1}^{(1)},	\bm{{X}}_{1}^{(2)},\dots,	\bm{{X}}_{1}^{(N)}\}$ from the training set, label matrix $\bm{Y}$, network length $L$ (including one MLP as last layer), activation functions $g_1,\dots,g_{L-1}$, label encoders $\bm{Q}_{1},\dots,\bm{Q}_{L-1}$}, and regularization factor $\gamma$.\\
	\For{$l \leftarrow 1$ \KwTo $L$}{\small
		\eIf{$l<L$ (CNN layers)}{
			\nl\textbf{LE process}:  obtain encoded label $\bm{\bar Z}_{l}$ with \eqref{eq_encoder};\\
			\nl\textbf{Feature flatten}: if needed, obtain 2-D feature matrix $\bm{\mathcal{X}}_{l}^{(n)}$ ($n=1,\dots,N$) through \eqref{eq_feature_flatten}-\eqref{eq_feature_flatten3};\\
			\nl\textbf{LS solution}: obtain CNN weight estimation $\bm{\mathcal{\hat W}}_{l}$ with \eqref{eq_acnn};\\
			\nl Calculate the activation $\bm{\mathcal{X}}_{l+1}^{(n)}$ ($n=1,\dots,N$) using \eqref{eq_activation} (if needed, calculate $\bm{X}_{l+1}^{(n)}$ using \eqref{eq_activation2tradition});\\
		}
		{
			\nl	obtain flattened feature matrix $\bm{\mathfrak{X}}_{l}$ using \eqref{eq_flatten_mlp};\\
			\nl calculate MLP weight estimation $\bm{\mathcal{\hat W}}_{l}$ using \eqref{eq_weight_mlp};
		}
	}
	\textbf{Outputs:} $\bm{\hat \mathcal{W}}_{1}, \dots, \bm{\hat \mathcal{W}}_{L}$
	\caption{ACnnL}
	\label{algo_An_CNN}
\end{algorithm}

\section{Analysis}
In this section, a theoretical analysis of the proposed ACnnL is given. The representation capacity of the proposed ACnnL is discussed. We shall provide evidence regarding ACnnL's potential to perfectly fit the training samples given sufficient parameters. More importantly, we show, from the structural point of view, that the CNN structure can be treated as a generalized MLP, but differs in that the CNN is more heavily regularized and therefore has better generalization ability.

\subsection{Discussion of Representation Capability}
Given sufficient parameters, the representation capacity can be shown by proving that $\bm{\mathcal{\hat W}}_{l}$ obtained using \eqref{eq_acnn_MP} can perfectly fit the set of training samples $\{\{\bm{\mathcal{X}}_{l}^{(1)},\bm{\bar Z}_{l}^{(1)}\},\{\bm{\mathcal{X}}_{l}^{(2)},\bm{\bar Z}_{l}^{(2)}\},\dots,\{\bm{\mathcal{X}}_{l}^{(N)},\bm{\bar Z}_{l}^{(N)}\}\}$. To this end, we first rewrite \eqref{eq_acnn_MP} as
\begin{align}\label{eq_acnn_rewrite}
	\bm{\mathcal{\hat W}}_{l} = \bm{\mathfrak{X}}^{\dagger} \bm{\mathfrak{Z}}
\end{align}
where 
\begin{align}\label{eq_cat_datamatrix}
	\bm{\mathfrak{X}} = \begin{bmatrix}
		\bm{\mathcal{X}}_{l}^{(1)}\\
		\bm{\mathcal{X}}_{l}^{(2)}\\
		\vdots\\
		\bm{\mathcal{X}}_{l}^{(N)}
	\end{bmatrix},
	\bm{\mathfrak{Z}} = \begin{bmatrix}
		\bm{\bar Z}_{l}^{(1)}\\
		\bm{\bar Z}_{l}^{(2)}\\
		\vdots\\
		\bm{\bar Z}_{l}^{(N)}
	\end{bmatrix}
\end{align}
in which $\bm{\mathfrak{X}}\in\mathbb{R}^{N(W_{l}-K_{l}+1)(H_{l}-K_{l}+1)\times C_{l}K_{l}^{2}}$. Subsequently, we can summarize the representation capacity in the following theorem.

\begin{thm}
	At layer $l$, If $\bm{\mathfrak{X}}$ has full row rank, the convolution layer can perfectly fit the learning samples, i.e., $\bm{\mathcal{\hat W}}_{l}$ obtained from \eqref{eq_acnn_MP} allows the objective function (with $\gamma=0$) in \eqref{eq_multiple} to become 0.
\end{thm}
\begin{proof}
	Having known that the solution in \eqref{eq_acnn_MP} can be equivalently rewritten by \eqref{eq_acnn_rewrite}, after the training, the objective function in \eqref{eq_multiple} can be replaced by
	\begin{align}\nonumber
		\norm{\bm{\mathfrak{Z}} - \bm{\mathfrak{X}}\bm{\mathcal{\hat W}}_{l}} &= \norm{\bm{\mathfrak{Z}} - \bm{\mathfrak{X}}\bm{\mathfrak{X}}^{\dagger}\bm{\mathfrak{Z}}}\\\label{eq_error}
		&= \norm{ (\bm{I}-\bm{\mathfrak{X}}\bm{\mathfrak{X}}^{\dagger})\bm{\mathfrak{Z}}}.
	\end{align}
	If $\bm{\mathfrak{X}}$ has full row rank, naturally we have $\bm{\mathfrak{X}}\bm{\mathfrak{X}}^{\dagger}=\bm{I}$. This allows a minimum of 0 for the objective function in \eqref{eq_error}.
	
\end{proof}

Clearly, to generate a matrix with full row rank, a necessary condition is that $\bm{\mathfrak{X}}$ must be a ``fat'' matrix by choosing appropriate channel size $C_{l}$ and kernel size $K_{l}$ such that 
\begin{align}\label{eq_full_rank_necessary}
	C_{l}K_{l}^{2}\ge N(W_{l}-K_{l}+1)(H_{l}-K_{l}+1).
\end{align}
That is, sufficiently large $C_{l}$ and $K_{l}$ could allow $\bm{\mathfrak{X}}$ to be of full rank, thereby leading to a perfect fitting of training samples. From a linear regression point of view, a larger $C_{l}$ or $K_{l}$ could potentially increase the fitting performance as the rank of $\bm{\mathfrak{X}}$ could increase even when \eqref{eq_full_rank_necessary} is not yet satisfied.

\subsection{Interpretation of CNN Structure through Regularization}\label{subsection_regularization}
As shown in \eqref{eq_acnn_MP}, the training of CNN weights boils down to solving a linear matrix equation by inverting the data matrix $\bm{\mathfrak{X}}\in\mathbb{R}^{N(W_{l}-K_{l}+1)(H_{l}-K_{l}+1)\times C_{l}K_{l}^{2}}$ given a total of $N$ samples. This closely resembles the MLP training measure (e.g., see \eqref{eq_weight_mlp_MP} or \cite{pil-first1995exact,cpnet2021}). Such resemblance is natural given the fact that the analytic learning can only be achieved by somehow converting the network learning problem into a linear one.

There are certainly distinctions between the proposed ACnnL and its MLP counterpart. The most representative one can be described by the extent to which the two training methods are regularized. This can be elaborated from a linear regression point of view as follows. 

Recall that one data sample becomes a matrix after transformation, e.g., $\bm{\mathcal{X}}_{l}\in\mathbb{R}^{C_{l}K_{l}^{2}\times(W_{l}-K_{l}+1)(H_{l}-K_{l}+1)}$. This differentiates it from the existing MLP-based analytic learning methods (e.g., see \eqref{eq_weight_mlp_MP} and \cite{pil-first1995exact,cpnet2021}) where one sample is presented by one row vector only. We take the MLP weight learning in \eqref{eq_mlp} as an example where such a solution can be interpreted by solving a regression problem constrained by $N$ linear conditions. The ACnnL trains CNN layers by solving \eqref{eq_multiple} yielding a solution given by \eqref{eq_acnn_rewrite}. In other words, the learning is through solving a linear problem that is constrained by a total of $N(W_{l}-K_{l}+1)(H_{l}-K_{l}+1)$ linear equations (i.e., $\bm{\mathfrak{X}}$ has $N(W_{l}-K_{l}+1)(H_{l}-K_{l}+1)$ rows as shown in \eqref{eq_cat_datamatrix}). That is, the number of linear constraints has increased by $(W_{l}-K_{l}+1)(H_{l}-K_{l}+1)$ times. For small kernel size $K_{l}$ in a large image (i.e., large $W_{l}$ and $H_{l}$), the increase can be tremendous. Hence, we may interpret the ACnnL as a more heavily regularized version than its MLP counterpart. This is why CNNs usually generalize better than MLPs. 

Conclusion made in \cite{zhang2016understanding} claims that the network structure (e.g., the CNN) is a form of implicit regularization, with no further theoretical findings provided. However, such a claim can be supported by our analysis here. Rather encouragingly, we provide evidence from a linear regression angle to validate that the structure is indeed a form of regularization. Specifically, the CNN is a form of regularized MLP by introducing more localized linear constraints.

Although the above explanation of generalization is facilitated by comparing different analytic learning methods, it should, to a reasonable extent, apply universally. As shown in \eqref{eq_cnn_matirx_multiplication}, the CNN layer can be converted into a linear matrix multiplication form. Hence, it can be treated as a generalized version of MLP with the sole difference being the number of feature vectors/constraints (e.g., the number of rows in $\bm{\mathcal{X}}_{l}$). Such a structural resemblance does not change whether or not the proposed learning scheme ACnnL is in play.
%

\section{Experiments}
To evaluate the performance of the proposed training strategy, we train CNNs to conduct classification tasks on several datasets, including MNIST, FashionMNIST, CIFAR-10 and CIFAR-100. The performance of ACnnL is assessed in comparison with that obtained by BP and by other non-BP methods. The assessment mainly includes the generalization performance and the time consumed to complete the training. The detailed comparisons are conducted mainly between BP and ACnnL as other methods cannot train CNNs.

\noindent\textbf{Datasets.} The MNIST and FashionMNIST have 10 classes with 50,000 images of $28\times28$ gray pixels for training and 10,000 for testing. The CIFAR-10 and CIFAR-100 datasets include $32\times32$ color images, with 50,000 images for training and 10,000 for testing. The CIFAR-10 and CIFAR-100 have 10 classes and 100 classes respectively.

The conducted experiments aim to validate the proposed ACnnL as the first analytic learning tool for training CNNs. To avoid complicating the validation, here we only train vanilla CNN structures (LeNet or VGG-like structures shown in next subsection) without involving advanced modules such as batch normalization \cite{batchnorm2015} or dropout \cite{dropout2014}. For the BP algorithm, we adopt 3 optimizers as follows:
\begin{itemize}
	\item Vanilla SGD, i.e., \textbf{BP (VSGD)}, without momentum or weight decay.
	\item SGD, i.e., \textbf{BP (SGD)}, with a momentum of 0.9 and weight decay of $5\times10^{-4}$.
	\item Adam, i.e., \textbf{BP (Adam)}, a modified optimizer proposed in \cite{adam2015ICLRposter}.
\end{itemize}

The network is trained for 100 epochs. The learning rate begins at 0.001 and is divided by 10 at 30, 60 and 80 epochs. During the experiments, no data augmentation techniques are applied  for BP and ACnnL. The generalization performance is measured by the testing accuracy, i.e., the accuracy on the testing set, after the training process. For iteration-based methods (e.g., BP), the testing accuracy is reported at the last epoch. No validation set is adopted. All experiments of BP and ACnnL are reported by the average of 3 runs on a workstation with Intel Xeon W-3265 Processor with 256G RAM and 2080 Ti GPU with 11G RAM.

\subsection{Vanilla Convolutional Neural Network Structure}
We experiment on a 5-layer\footnote{For convenience, the layer count here only refers to a total of trainable layers (e.g., CNN and MLP) excluding other types of layers such as pooling.} vanilla CNN containing 4 CNN layers followed by one MLP.  For convenience, the structure is represented by {\{\bf [Conv($K_{1}\times K_{1}$)$\times C_{1}$ -- $g_{1}$] --- [AvgPool($P\times P$)] --- [Conv($K_{2}\times K_{2}$)$\times C_{2}$ -- $g_{2}$] -- [AvgPool($P\times P$)] --- [Conv($K_{3}\times K_{3}$)$\times C_{3}$ -- $g_{3}$] --- [Conv($K_{4}\times K_{4}$)$\times C_{4}$ -- $g_{4}$] --- [MLP($K$)]\}}. Specifically, {\bf [Conv($K_{l}\times K_{l}$)$\times C_{l}$ -- $g_{l}$]} indicates a $K_{l}\times K_{l}$ CNN layer with $C_{l}$ channels followed by an activation function $g_{l}$. {\bf [AvgPool($P\times P$)]} represents the average pooling of size $P\times P$ with stride $P$ for down-sampling the feature map. {\bf MLP($K$)} means the MLP layer mapping the current feature to a vector of $K$ dimensions. For the classification problems in this experiment, $K$ is the number of classes.

We use a consistent set of structural parameters for various datasets. Specifically, we set $K_{1}=5$, $K_{2},K_{3},K_{4}=3$, $C_{2},C_{3}, C_{4}=2C_{1}, 4 C_{1}, 4C_{1}$ and $P=2$ respectively. The activation functions are chosen to be LeakyReLU with slope $0.1$ for negative activations. We name this unified structure ``CNN-5(C)'', representing a 5-layer vanilla CNN with a tunable channel parameter $C=C_{1}$. For instance, a ``CNN-5(C=16)'' structure for MNIST indicates a CNN with a structure of {\{\bf [Conv($5\times 5$)$\times 16$ -- LeakyReLU] --- [AvgPool($2\times 2$)] --- [Conv($3\times 3$)$\times 32$ -- LeakyReLU] -- [AvgPool($2\times 2$)] --- [Conv($3\times 3$)$\times 64$ -- LeakyReLU] --- [Conv($3\times 3$)$\times 64$ -- LeakyReLU] --- [MLP($10$)]\}}.

\subsection{Generalization Performance Evaluation and Comparison}

\subsubsection{ACnnL's generalization w.r.t. the selection of $\gamma$}
The solution in ACnnL is constructed based on the LS technique. The regularization factor $\gamma$ could play an important role in determining the generalization ability. Hence, we explore the selection of $\gamma$ with a range from $10^{-5}$ to $10^{5}$ to evaluate its impact. As shown in TABLE \ref{table_gamma}, the ACnnL with a regularization factor $\gamma$ of a value around $10^{2}$ gives satisfying generalization performance. A too large or too small $\gamma$ leads to deteriorated generalization. For convenience, we adopt $\gamma=10^{2}$ in the following experiments.

\begin{table}
	\centering
	\caption{Testing accuracy w.r.t. $\gamma$ settings.}
	\resizebox{1\linewidth}{!}{\begin{tabular}{c|lllllllll}
			\toprule
			\hline
			&\multicolumn{7}{c}{CNN-5(C=128), $\gamma=?$}\\
			\hline
			& $10^{-5}$& $10^{-3}$& $10^{-2}$& $10^{-1}$& $1$& $10$& $10^{2}$& $10^{3}$& $10^{5}$\\
			\hline
			MNIST&0.86188	&0.79274	&0.89666	&0.94736	&0.90952	&0.94216	&0.99024	&\textbf{0.99032}	&0.92568\\
			FashionMNIST&0.57248	&0.56752	&0.60428	&0.55812	&0.67224	&0.69866	&\textbf{0.89718}	&0.89122	&0.71300\\
			CIFAR-10&0.21044	&0.2864	&0.26074	&0.27988	&0.28672	&0.2044	&\textbf{0.63984}	&0.5958	&0.21228\\
			CIFAR-100&0.1063	&0.0739	&0.0137	&0.0851	&0.0808	&0.0089	&\textbf{0.3717}	&0.2518	&0.01\\
			\hline
			\bottomrule
	\end{tabular}}
	\label{table_gamma}
\end{table}

\subsubsection{Overall Generalization Comparison}
Here we give an overall evaluation of ACnnL's generalization in comparison with that of BP as well as the existing non-BP methods. The non-BP methods included for comparison are MLP-based analytic learning techniques (i.e., KARnet \cite{karnet2018}, PILAE \cite{pilae2018s}, ANnet \cite{annet2018} and CPNet \cite{cpnet2021}) and DF \cite{deepforest2017}.

As reported in TABLE \ref{table_main_gen_comparison}, the ACnnL gives better accuracies than DF's on MNIST (e.g., 0.9931 v.s. 0.9926) and CIFAR-10 (e.g., 0.7049 v.s. 0.6337). Among the compared analytic learning methods, the proposed ACnnL significantly outperforms its MLP counterparts in terms of generalization. The performances deviate more evidently on relatively difficult datasets (e.g., CIFAR-10) than easy ones (e.g., MNIST). For instance, on MNIST the ACnnL gains a $1\%$ lead over the ANnet while the lead becomes $20\%$ on CIFAR-10 (e.g., ACnnL of $0.7049$ v.s. ANnet of $0.4990$). This shows that the ACnnL is an effective CNN trainer with analytical solutions. It also partly demonstrates that the CNN structure is in general more powerful than the MLP for dealing with image-based classification problems.

It is observed that MLP-based analytic learning methods suffer much from over-fitting more severely than the proposed ACnnL. For instance, on CIFAR-10 the ANnet obtains a training accuracy of $0.7599$ but receives a testing accuracy of only $0.4990$. The proposed ACnnL achieves a training accuracy of $0.8019$ with a rather comparable prediction accuracy of $0.6706$. This also validates our claim in Section \ref{subsection_regularization}. That is, the ACnnL trains neural networks by imposing significantly more constraints than those in MLP-based methods, thereby heavily regularizing the optimization for enhancing network generalization.

On the other hand, as an iteration-free training strategy, it is of obligation to investigate the prediction performance deviation from the commonly used BP. As indicated in TABLE \ref{table_main_gen_comparison}, CNNs trained by BP (Adam) give better generalization than those trained by the proposed ACnnL. This is reasonable as the ACnnL forgoes the global feedback information and focuses on the local one in order to achieve analytic learning. The lower layers in the CNN trained by the ACnnL do not receive feedbacks from the upper layers, leading to certain under-fitting. However, it is encouraging to see that the ACnnL outperforms the BP (VSGD) and BP (SGD), e.g., ACnnL with 0.6706 v.s. BP (VSGD) with 0.2358 and BP (SGD) with 0.5818 on CIFAR-10. This demonstrates that the ACnnL can simply bring out a network's generalization power with a very limited tuning of hyperparameters (e.g., $\gamma$). The BP needs to carefully select an appropriate optimizer (e.g., Adam in this case) before the network's generalization fully appears.

\begin{table}[h]
	\centering
	\caption{Generalization performance of networks trained by various compared methods. Methods marked by $^{*}$ report their best results from their papers with diverse structures.}
	\resizebox{1\linewidth}{!}{
		\begin{tabular}{c|c|c|c|c}
			\toprule
			\hline
			Dataset&  Structure&Training Method &Training Acc.  &  Test. Acc. \\
			\hline
			MNIST& MLP&KARnet$^{*}$ \cite{karnet2018} &- &0.9244\\
			MNIST& MLP&ANnet$^{*}$ \cite{annet2018}&0.9993&0.9824  \\
			MNIST& MLP&PILAE$^{*}$  \cite{pilae2018s}&- &0.9692\\
			MNIST& MLP&RBM-GI$^{*}$ \cite{analytic-learning-finallayer2019noniterative} &- &0.9758\\
			MNIST& MLP&CPNet$^{*}$ \cite{cpnet2021}&- &0.9850\\
			MNIST& DF&DF$^{*}$ \cite{deepforest2017}&-&0.9926\\
			MNIST& CNN-5(C=256)&BP (VSGD)\cite{bp1974phd} & 1.00& 0.8991 \\
			MNIST& CNN-5(C=512)&BP (VSGD)\cite{bp1974phd} & 1.00& 0.9147 \\
			MNIST& CNN-5(C=256)&BP (SGD)\cite{bp1974phd} & 1.00& 0.9832 \\
			MNIST& CNN-5(C=512)&BP (SGD)\cite{bp1974phd} & 1.00& 0.9801 \\
			MNIST& CNN-5(C=256)&BP (Adam)\cite{bp1974phd} & 1.00& 0.9941 \\
			MNIST& CNN-5(C=512)&BP (Adam)\cite{bp1974phd} & 1.00& \textbf{0.9943} \\
			\rowcolor{gray!30}
			MNIST& CNN-5(C=256)&ACnnL& 0.9945&0.9920\\
			\rowcolor{gray!30}
			MNIST& CNN-5(C=512)&ACnnL& 0.9989&0.9931\\
			\hline
			CIFAR-10& MLP&ANnet$^{*}$ \cite{annet2018}&0.7599  &  0.4990\\
			CIFAR-10& DF&DF$^{*}$ \cite{deepforest2017}&-&0.6337\\
			CIFAR-10& CNN-5(C=256)&BP (VSGD) \cite{bp1974phd} & 0.5806& 0.2358 \\
			CIFAR-10& CNN-5(C=256)&BP (SGD) \cite{bp1974phd} & 0.8088& 0.5818 \\
			CIFAR-10& CNN-5(C=256)&BP (Adam) \cite{bp1974phd} & 1.00& \textbf{0.7309} \\
			\rowcolor{gray!30}
			CIFAR-10& CNN-5(C=256)&ACnnL& 0.8019&0.6706 \\
			\rowcolor{gray!30}
			CIFAR-10& CNN-5(C=512)&ACnnL& 0.9077&0.7049 \\
			\hline
			\bottomrule
	\end{tabular}}
	
	\label{table_main_gen_comparison}
\end{table}

\subsubsection{ACnnL performance w.r.t. network width}
We also evaluate the generalization performance of our ACnnL w.r.t. network width. This is done by training CNN-5(C) with various $C$ values ranging from 16 to 512. As shown in TABLE \ref{table_width}, the network's generalization performance increases with a wider network structure. Such a performance gain is not surprising as the analytic learning is less likely to experience over-fitting when increasing the number of parameters \cite{cpnet2021}. Instead, the network desires more parameters to enhance its generalization. This is also consistent with the trend of traditional deep learning \cite{zagoruyko2016wide}.

\begin{table}
	\centering
	\caption{Testing accuracy w.r.t. network width.}
	\resizebox{1\linewidth}{!}{\begin{tabular}{c|cccccc}
			\toprule
			\hline
			&\multicolumn{6}{c}{CNN-5(C=?)}\\
			\hline
			& C=16& C=32& C=64& C=128& C=256& C=512\\
			\hline
			MNIST&0.9739&0.9838
			&0.9878
			&0.9903&0.9920&\textbf{0.9931}\\
			FashionMNIST&0.8466
			&0.8708&0.8844&0.8980&0.9100&\textbf{0.9155}\\
			CIFAR-10&0.4589
			&0.5272
			&0.5809&0.6345&0.6706&\textbf{0.7049}\\
			CIFAR-100&0.1144 &0.1783&0.2532&0.3338&0.4041&\textbf{0.4628}\\
			\hline
			\bottomrule
	\end{tabular}}
	\label{table_width}
\end{table}

\subsubsection{ACnnL performance w.r.t. network depth}
In particular, we evaluate the impact of a network's depth on the generalization performance. In TABLE \ref{table_depth}, the evolution of generalization performance w.r.t. network depth is reported by training CNNs of various depths on MNIST and CIFAR-10. A clear pattern of an increasing performance for a deeper structure is observed in the reported numbers. This is optimistic as our analytic method preserves the most important feature of current deep learning. That is, features extracted by deeper CNNs are usually more discriminative than those obtained by shallow ones.

\begin{table*}
	\centering
	\caption{Evaluation of the proposed ACnnL in terms of network depth.}
	\resizebox{1\textwidth}{!}{
		\begin{tabular}{c|cl|l}
			\toprule
			\hline
			Dataset& \multicolumn{2} {c|} {\bfseries Structure}&Test. Acc.\\
			\hline
			MNIST&CNN-2(C=128)&{\{\bf [Conv($5\times 5$)$\times 16$ -- LeakyReLU] --- [MLP($10$)]\}}.&0.8949\\
			MNIST&CNN-3(C=128)&{\{\bf [Conv($5\times 5$)$\times 16$ -- LeakyReLU] --- [AvgPool($2\times 2$)] --- [Conv($3\times 3$)$\times 32$ -- LeakyReLU] --- [MLP($10$)]\}}&0.9563 $\uparrow$\\
			MNIST&CNN-4(C=128)&{\{\bf [Conv($5\times 5$)$\times 16$ -- LeakyReLU] --- [AvgPool($2\times 2$)] --- [Conv($3\times 3$)$\times 32$ -- LeakyReLU] --- [AvgPool($2\times 2$)] --- [Conv($3\times 3$)$\times 64$ -- LeakyReLU]  --- [MLP($10$)]\}}.&0.9881 $\uparrow$\\
			MNIST&CNN-5(C=128)&{\{\bf [Conv($5\times 5$)$\times 16$ -- LeakyReLU] --- [AvgPool($2\times 2$)] --- [Conv($3\times 3$)$\times 32$ -- LeakyReLU] --- [AvgPool($2\times 2$)] --- [Conv($3\times 3$)$\times 64$ -- LeakyReLU] --- [Conv($3\times 3$)$\times 64$ -- LeakyReLU] --- [MLP($10$)]\}}.&0.9900 $\uparrow$\\
			\hline
			CIFAR-10&CNN-2(C=128)&{\{\bf [Conv($5\times 5$)$\times 16$ -- LeakyReLU] --- [MLP($10$)]\}}&0.4953\\
			CIFAR-10&CNN-3(C=128)&{\{\bf [Conv($5\times 5$)$\times 16$ -- LeakyReLU] --- [AvgPool($2\times 2$)] --- [Conv($3\times 3$)$\times 32$ -- LeakyReLU] --- [MLP($10$)]\}}&0.5353 $\uparrow$\\
			CIFAR-10&CNN-4(C=128)&{\{\bf [Conv($5\times 5$)$\times 16$ -- LeakyReLU] --- [AvgPool($2\times 2$)] --- [Conv($3\times 3$)$\times 32$ -- LeakyReLU] --- [AvgPool($2\times 2$)] --- [Conv($3\times 3$)$\times 64$ -- LeakyReLU]  --- [MLP($10$)]\}}.&0.6046 $\uparrow$\\
			CIFAR-10&CNN-5(C=128)&{\{\bf [Conv($5\times 5$)$\times 16$ -- LeakyReLU] --- [AvgPool($2\times 2$)] --- [Conv($3\times 3$)$\times 32$ -- LeakyReLU] --- [AvgPool($2\times 2$)] --- [Conv($3\times 3$)$\times 64$ -- LeakyReLU] --- [Conv($3\times 3$)$\times 64$ -- LeakyReLU] --- [MLP($10$)]\}}.&0.6277 $\uparrow$\\
			\hline
	\end{tabular}}
	\label{table_depth}
\end{table*}

\begin{figure}
	\centering
	\includegraphics[width=1\linewidth]{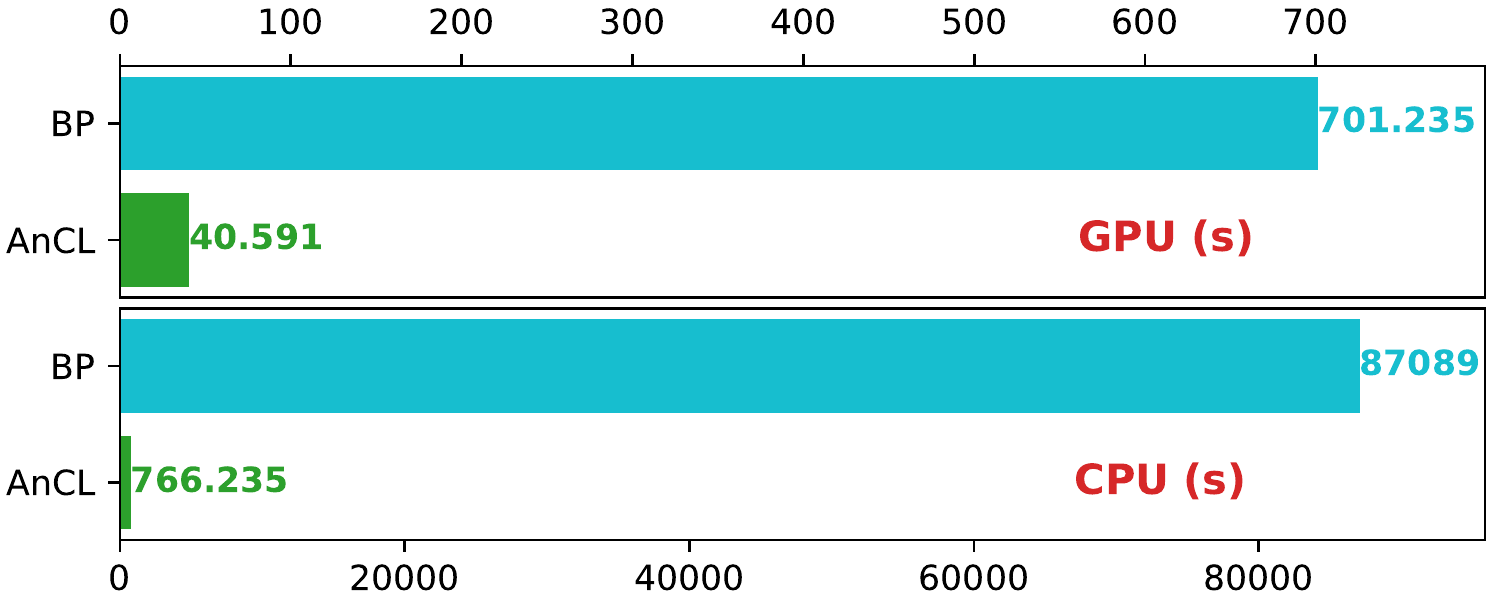}
	\caption{Running time comparison for training CNN-5(C=128).}
	\label{fig:time}
\end{figure}

\subsubsection{Speed evaluation}
The ACnnL trains CNNs with merely a single epoch leading to an extremely fast training speed. To demonstrate this, we measure the running times of BP and ACnnL for training CNN5(C=128) with CPU and GPU settings respectively . As shown in Fig. \ref{fig:time}, training the network with GPU, our ACnnL spends around 40 seconds to complete the training, which is more than \textbf{17}$\times$ faster than that of BP (more than 700 seconds for 100 epochs). This gap of speed grows wider for the CPU setting where the ACnnL runs \textbf{113}$\times$ faster than the BP (ACnnL with 766 seconds v.s. BP with 87089 seconds)! Such a speed growth can be justified by the fact that GPUs cannot accelerate matrix inverse in the same way they accelerate matrix multiplication. In fact, comparing the speed using GPUs is less fair as GPUs are mainly built for BP-based iterative algorithms while matrix inverse is yet to be well supported. From the cost-efficient point of view, the proposed ACnnL obtains a reasonable trade-off between the accuracy (e.g.,  a 6\% drop from 0.73 to 0.67 for CNN-5(C=256)) and the time consumption (e.g., 17$\times$ to 113$\times$ speed increase).

The ACnnL experiences a marginal generalization degradation compared with the BP. However, it is motivating to see that training a network for merely one epoch could achieve most of its generalization potential. Yet, our goal is not to replace the BP-based iteration methods. Instead, the ACnnL is more of an alternative to trade accuracy for speed in a reasonable manner for certain scenarios where computing resources are limited (e.g., no available GPUs) or fast training is required (e.g., in near real-time training applications). The advantage of high training speed and interpretability is very encouraging.

More importantly, our goal does not stop at accuracy-speed trade-off. Instead, it is of great interest to incorporate the ACnnL with the existing BP such that a comparable generalization with that of BP can be obtained with the network trained for much fewer epochs. For instance, one might adopt the ACnnL as an initialization strategy and the BP as a subsequent adjustment (e.g., BP as fine-tuning). However, currently this incorporation is not trivial due to numerical differences. For instance, CNN weights trained with BP are usually designed to be small while the weights trained with ACnnL in the current phase exhibit a relatively random pattern (e.g., some weights could have large values) resulted from the matrix inversion. Connecting the ACnnL and the BP has potential and is to be investigated in future work.

\subsection{Small-sample Performance}\label{subsection_small_samples}
In this subsection, we conduct experiments with small-sample settings. Instead of utilizing all the samples from the datasets, we randomly select $N_{c}$ (with the total samples $KN_{c}\le N$ smaller than the original size) samples from each class and train the classifier from scratch using these selected samples. The small-sample experiment resembles but differs from the recently popular few-shot learning as no ``support set'' \cite{yue2020interventional} is provided to pre-train the network. Therefore, we exclude the existing few-shot learning methods for comparison, and focus the evaluation against the BP to highlight the impact of analytic learning.

We map the evolution of testing accuracy w.r.t. $N_{c}$ on CIFAR-10 and MNIST in Fig. \ref{fig:small-sample}. As shown in the figure, the proposed ACnnL outperforms the BP for small $N_{c}$ values. For instance, results on MNIST in Fig. \ref{fig:small-sample}(a) indicate that networks trained with ACnnL give better generalization for $N_{c}< 1000$ samples. This is consistent with the pattern on CIFAR-10 in Fig. \ref{fig:small-sample}(a). Empirically, it is evident that our ACnnL gives better generalization performance than that of BP (here we adopt Adam optimizer as it gives the best performance among the 3 optimizers) in small-sample setting. Theoretically, this can also be explained by the ``over-fitting-avoiding'' nature of the LS technique that plays a major role in the ACnnL.

The ACnnL having outstanding prediction performance is strategically advantageous to recent deep learning developments focusing on small-sample related tasks (e.g., few-shot learning). This trend is natural as the heart of deep learning has been hinged upon availability of sufficient data. As indicated in the upper panel of Fig. \ref{fig:small-sample}, we can see that networks trained by BP experience a much sharper climb of generalization with $N_{c}>100$ than that with small $N_{c}$ values. For the BP-trained networks, there appear to be a threshold of sample size, below which the generalization would be heavily suppressed. Contrarily, the proposed ACnnL does not share the same weakness because the LS-based solution is less likely to invite over-fitting. This allows our method to become of value with reasonable accuracy in applications for scarce data (e.g., medical images that are expensive).

\begin{figure}
	\centering
	\includegraphics[width=1\linewidth]{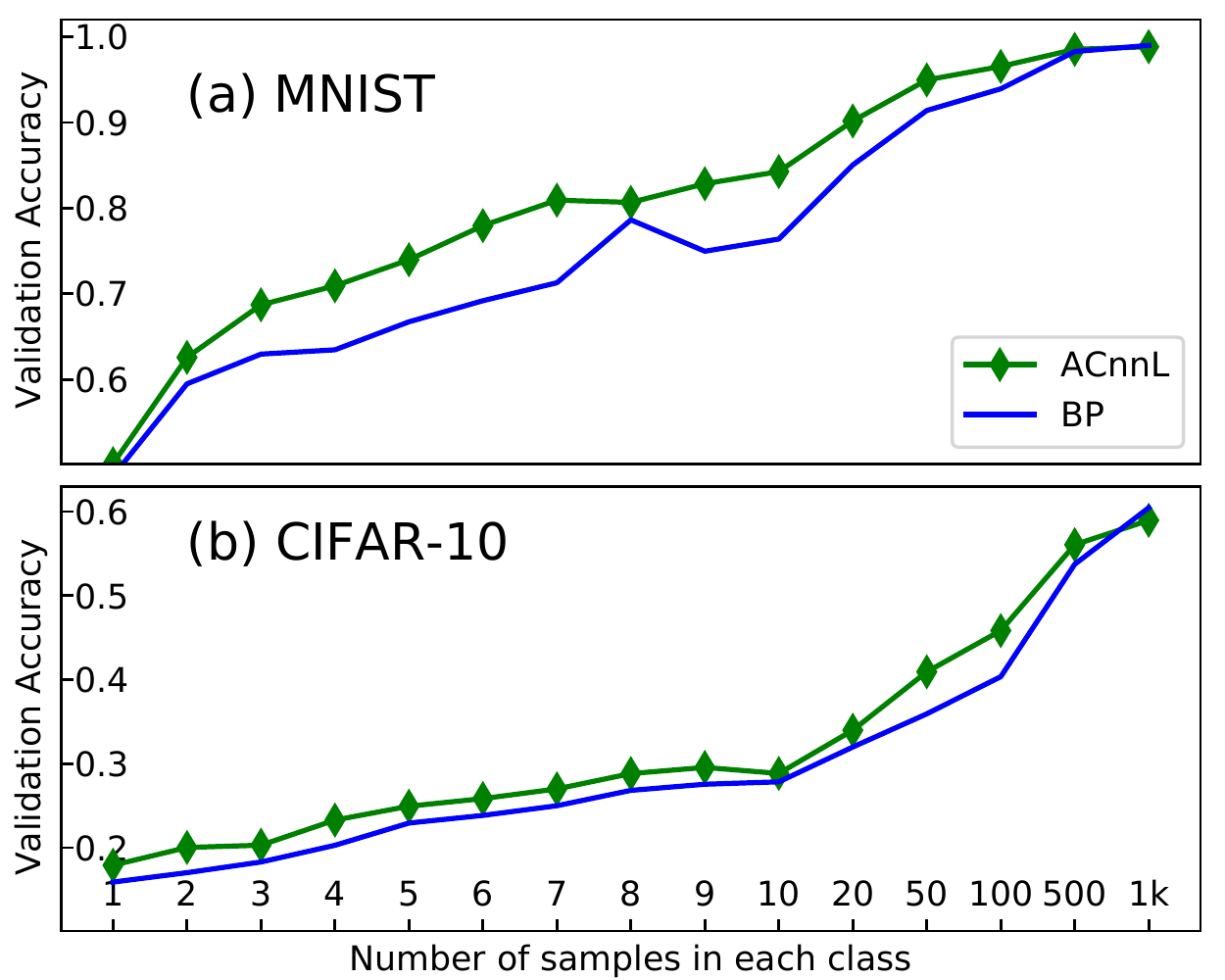}
	\caption{Small-sample accuracy performance of ACnnL in comparison with BP (Adam) on (a) MNIST and (b) CIFAR-10 (``1k'' indicates 1000).}
	\label{fig:small-sample}
\end{figure}

To summarize the empirical evidence from the above experiments, we highlight the insights as follows:
\begin{itemize}
	\item Among the studied analytic learning methods, the ACnnL leads to outstanding generalization performance (e.g., from $50\%$ to $70\%$ on CIFAR-10).
	\item The ACnnL shows consistent patterns with the existing deep learning trend. For instance, the generalization ability increases with a deeper or wider structure.
	\item The ACnnL experiences mild degradation of generalization in comparison with its BP counterpart (e.g., BP (Adam)), but receives a significant boost of training speed due to the one-epoch training style.
	\item The ACnnL is advantageous for training small-sample datasets from scratch. This is empirically and theoretically true as the LS technique establishes a major component.
\end{itemize}

\section{Conclusion}
In this paper, we have proposed an analytic convolutional neural network learning method called ACnnL, the first analytic method for training CNNs. The one-epoch training style of ACnnL completes the training in a significantly fast manner while only experiencing a mild generalization degradation comparing with that trained using the well-known back-propagation. Such a speed-generalization trade-off is consistent with that in the existing analytic learning community. Essentially, with the derivation of ACnnL, we have theoretically uncovered why CNNs generalize better than MLPs in terms of imposition of localized constraints by solving linear matrix functions. The experiments have validated our proposed training scheme, showing a significantly improved generalization over its analytic MLP counterpart. In addition, the ACnnL has been empirically demonstrated to excel in small-sample scenarios, which bears great potential in data-scarce applications.

\section*{Acknowledgments}
This work was supported in part by the Science and Engineering Research Council, Agency of Science, Technology and Research, Singapore, through the National Robotics Program under Grant 1922500054, and in part by the Basic Science Research Program through the National Research Foundation of Korea funded by the Ministry of Education, Science and Technology under Grant (NRF-2021R1A2C1093425).

\bibliographystyle{IEEEtran}
\bibliography{acnn}

%
\end{document}